\pdfoutput=1

\documentclass[11pt]{article}

\usepackage{authblk}
\usepackage{acl}

\usepackage{times}
\usepackage{latexsym}

\usepackage[T1]{fontenc}

\usepackage[utf8]{inputenc}

\usepackage{microtype}

\usepackage{todonotes}

\usepackage{amsmath}
\usepackage{amssymb}
\usepackage{amsthm}
\usepackage[ruled,vlined]{algorithm2e}
\usepackage{booktabs}
\usepackage{blindtext}
\usepackage{graphicx}
\usepackage{xfrac}
\usepackage{multirow}
\usepackage{enumitem}
\usepackage{arydshln}

\DeclareMathOperator{\E}{\mathbb{E}}
\DeclareMathOperator{\Var}{Var}
\DeclareMathOperator{\relu}{ReLU}
\DeclareMathOperator{\sparsemax}{sparsemax}
\DeclareMathOperator{\softmax}{softmax}
\DeclareMathOperator{\h}{H}
\DeclareMathOperator{\diag}{diag}
\DeclareMathOperator{\vect}{vec}
\DeclareMathOperator{\entmax}{\alpha-entmax}
\DeclareMathOperator{\arelu}{\alpha-ReLU}

\newcommand{\zh}[1]{{\color{blue}{#1}}}
\newcommand{\wrong}[1]{{\color{red}{#1}}}
\newcommand{\corr}[1]{{\color{olive}{#1}}}

\newtheorem{lemma}{Lemma}

\title{Speeding Up Entmax}

\author[1]{Maxat Tezekbayev}
\author[2]{Vassilina Nikoulina}
\author[2]{Matthias Gall\'e}
\author[1]{Zhenisbek Assylbekov}

\affil[1]{Nazarbayev University}
\affil[ ]{\texttt{\{maxat.tezekbayev, zhassylbekov\}@nu.edu.kz}}
\affil[2]{NAVER LABS Europe}
\affil[ ]{\texttt{\{vassilina.nikoulina, matthias.galle\}@naverlabs.com}}

\begin{document}
\maketitle
\begin{abstract}
Softmax is the de facto standard for  normalizing logits in modern neural networks for language processing. 
However, by producing a  dense probability distribution each token in the vocabulary has a nonzero chance of being selected at each generation step, leading to a variety of reported problems in text generation. $\alpha$-entmax of \citet{DBLP:conf/acl/PetersNM19} solves this problem, but is 
unfortunately slower than softmax. 

In this paper, we propose an alternative to $\alpha$-entmax, which keeps its virtuous characteristics, but is as fast as optimized softmax and achieves on par or better performance in machine translation task. 
\end{abstract}

\section{Introduction}
Sparseness of vector representations is a desirable trait in neural network models for natural language processing (NLP): words (subwords) are discrete objects by their nature, and, accordingly, are encoded by one-hot embeddings at the input and output of neural networks. However, to predict a categorical response in neural models, softmax is most often used, which produces a \emph{dense} probability distribution, i.e. every category (word/subword) receives a non-zero probability.

Recent studies  suggest that it is this output density that poses problems when the trained NLP model is used for inference. For example, in the case of text generation, unconstrained sampling from a trained language model results in poor quality of the resulting text \cite{DBLP:conf/iclr/HoltzmanBDFC20}. In neural machine translation (NMT), exact decoding from a trained model often results in empty text \cite{DBLP:conf/emnlp/StahlbergB19}.\footnote{The authors called this phenomenon the \textit{cat got your tongue} problem.}
To get around these problems, constrained decoding techniques have been proposed, most of which artificially impose sparsity on softmax prediction. For example, \citet{DBLP:conf/acl/LewisDF18} propose to sample from the top-$k$ probable words, and \citet{DBLP:conf/iclr/HoltzmanBDFC20} propose to sample from the most probable words, which  comprise the cumulative probability $p$. While these methods are effective, they are ad-hoc solutions that lead to a mismatch between how the model is trained and how it is used at inference. 

In this regard, the works on sparse alternatives to softmax stand apart since they allow us to make inference from the model in the same way than it was trained. 
Some of the most successful and elegant solutions are sparsemax~\citep{DBLP:conf/icml/MartinsA16} and its generalization $\alpha$-entmax~\citep{DBLP:conf/acl/PetersNM19}. 
When coupled with suitable losses, these transformations are not inferior to softmax, and sometimes even surpass it as measured with final performance metrics on a number of tasks. 
A problem with these transformations however is that they are 
significantly 
slower than softmax when the number of categories (vocabulary size) is tens of thousands, as in the case of text generation.
This is because $\alpha$-entmax transformation---in its original formulation---requires sorting over the logits.\footnote{We also compare against an approximate version which only performs sorting on the highest values of the logits.}

In this work, we ask the question: \textit{is it possible to obtain a sparse output like that of $\alpha$-entmax, but without its degradation in computational speed?} Our answer is affirmative---we propose a sparse output transformation that
\begin{itemize}
    \item is on par or superior to softmax and $\alpha$-entmax in the NMT tasks,
    \item works as fast as softmax during training and at inference,
    \item gives the same training dynamics as $\alpha$-entmax (in training steps).
\end{itemize}
The most surprising thing is that such a transformation is simply a shifted ReLU raised to power $\frac{1}{\alpha-1}$, which we call $\boldsymbol\alpha$\textbf{-ReLU}.

The rest of the paper is organised as follows. In Sect.~\ref{sec:relu} we motivate the choice of $\alpha$-ReLU as the output transformation, and also select an appropriate loss function. In Sect.~\ref{sec:experiments} we experimentally confirm our claims about performance and  output speed of $\alpha$-ReLU in the NMT task. Sect.~\ref{sec:analysis} is devoted to a comparative analysis of $\alpha$-ReLU and $\alpha$-entmax in terms of sparsity, ability to solve the empty translation problem, and training dynamics.

\section{$\boldsymbol\alpha$-ReLU at Output}\label{sec:relu}
Our departure point is the $\alpha$-entmax transformation of \citet{DBLP:conf/acl/PetersNM19} which can be defined for $\mathbf{z}\in\mathbb{R}^d$ as
\begin{equation}
    \entmax_i(\mathbf{z})=[(\alpha-1)z_i-\tau(\mathbf{z})]_+^{\frac{1}{\alpha-1}},\label{eq:entmax}
\end{equation}
where $[x]_+:=\max\{x,0\}$, and $\tau:\mathbb{R}^d\to\mathbb{R}$ is the (unique) function that satisfies $\sum_j[(\alpha-1)z_j-\tau(\mathbf{z})]_+^{\frac{1}{\alpha-1}}=1$ for any $\mathbf{z}$. It is this threshold $\tau$ that makes the computation of $\alpha$-entmax slow, because one needs to sort the components of $\mathbf{z}$ to find $\tau$ \cite[Alg.~2]{DBLP:conf/acl/PetersNM19}.


As we can see, the threshold $\tau$ is only needed to ensure that  $\entmax(\mathbf{z})$ is a \emph{probability} distribution. 
We loosen this constraint, and only require  \emph{non-negative} weights, which is sufficient for most uses. 
Consider then a transformation
\begin{equation}
\arelu_i(\mathbf{z}):=[(\alpha-1)z_i-\tau]_+^{\frac{1}{\alpha-1}},
\end{equation}
where $\tau$ is a \textit{constant} that does not depend on $\mathbf{z}$. 
In order to
force $\arelu(\mathbf{z})$---applied to the logits $\mathbf{z}$---to converge to the one-hot vector $\mathbf{e}_y$ of the gold label $y$ we need to adjust the corresponding loss.
This can easily be done by feeding the logits $\mathbf{z}$ and the output $\arelu(\mathbf{z})$ into the following loss, which we call $\boldsymbol\alpha$\textbf{-ReLU loss}.
\begin{multline}
    \textstyle\ell(\mathbf{z},y)=(\arelu(\mathbf{z})-\mathbf{e}_y)^\top\left(\mathbf{z}-\frac{\tau}{\alpha-1}\mathbf{1}\right)\\+\h_\alpha[\arelu(\mathbf{z})],\label{eq:sp_loss}
\end{multline}
where $\h_\alpha[\mathbf{p}]:=\frac{1}{\alpha(\alpha-1)}\left(1-\sum_j p_j^\alpha\right)$, $\alpha\ne1$, is the Tsallis $\alpha$-entropy \cite{tsallis1988possible}, and $\mathbf{1}:=(1,\ldots,1)\in\mathbb{R}^d$ is a vector of ones. 
The rationale for coupling $\arelu$ with the loss \eqref{eq:sp_loss} is the following
\begin{lemma} \label{lem:gradient} For any $\tau\in\mathbb{R}$, the gradient of the $\alpha$-ReLU loss \eqref{eq:sp_loss} is given by
$$
\nabla_{\mathbf{z}}\ell(\mathbf{z},y)=\arelu(\mathbf{z})-\mathbf{e}_y.
$$
\end{lemma}
\begin{proof}
The proof is in Appendix~\ref{app:loss_proof}.
\end{proof}

By Lemma~\ref{lem:gradient},  gradient-based minimization of $\ell$ indeed forces $\arelu(\mathbf{z})\to\mathbf{e}_y$. Notice that this is similar to what happens when the softmax normalization is coupled with the cross-entropy loss or when $\alpha$-entmax is coupled with the entmax loss. In both cases differentiating the loss with respect to logits gives $\mathbf{p}-\mathbf{e}_y$, where $\mathbf{p}$ is either $\softmax(\mathbf{z})$ or $\entmax(\mathbf{z})$ \cite{DBLP:conf/icml/MartinsA16,DBLP:conf/acl/PetersNM19}.

\paragraph{Remark.} Recall that $\alpha$-entmax is a generalization of sparsemax. For example, 2-entmax is essentially sparsemax, and for $\alpha\in(1,2)$ we get a smoothed version of sparsemax. Similarly, $\alpha$-ReLU is a kind of generalization of ReLU. So, the standard ReLU is 2-ReLU (with $\tau = 0$), and for $\alpha\in (1,2)$ we get a smoothed ReLU (see Fig.~\ref{fig:arelu}).
\begin{figure}[htbp]
    \centering
    \includegraphics[width=.45\textwidth]{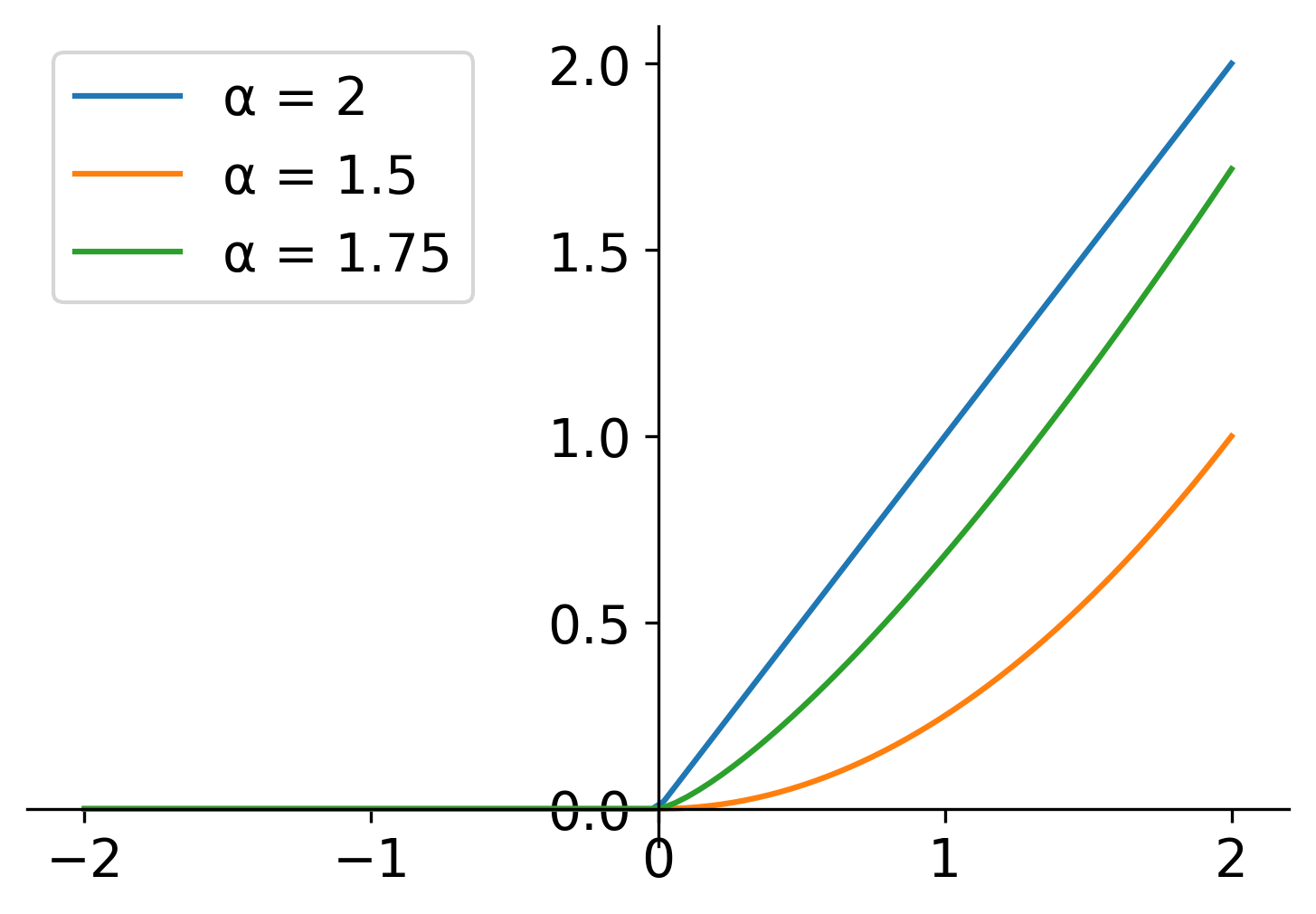}
    \caption{The graph of $\arelu(x)$ for several $\alpha\in(1,2]$, with $\tau=0$. 2-ReLU is a standard $\relu(x):=[x]_+$.}
    \label{fig:arelu}
\end{figure}

\begin{table*}[htbp]
    \centering
    \begin{tabular}{l l c c c}
    \toprule
        Output Transform & Loss & IWSLT De$\to$En & WMT En$\to$De & WMT En$\to$Ru \\
        \midrule
        softmax & cross-entropy & 35.3 & \textbf{28.7} & 22.4 \\
        sparsemax & sparsemax loss & 35.5 & 26.6 & 19.6 \\
        1.5-entmax & 1.5-entmax loss & 36.6 & 28.6 & 23.9\\ 
        1.5-entmax ($k=100$) & 1.5-entmax loss & 36.7 & 28.4 & 23.7 \\
        1.5-ReLU & 1.5-ReLU loss & \textbf{37.3} & 28.6 & \textbf{24.6} \\
        \midrule
        \multicolumn{2}{c}{\# Trainable parameters} & 47M & 75M & 75M \\
        \bottomrule
    \end{tabular}
    \caption{NMT results: comparison of softmax, sparsemax, 1.5-Entmax and the proposed 1.5-ReLU as the output transformations in the Transformer NMT model. Reported is detokenized test BLEU.}
    \label{tab:nmt_results}
\end{table*}

\section{Experiments}\label{sec:experiments}

In theory, nothing prevents $\alpha$-ReLU from learning what $\alpha$-entmax is learning. However, in practice we can have a different picture, because training is conditioned by many factors---the size of the dataset, the architecture of the neural network, the optimization algorithm, etc. In this section, we compare $\alpha$-ReLU empirically with $\alpha$-entmax (as well as with sparsemax and softmax), assuming all other factors are fixed. The goal of these experiments is to evaluate the consequences of using $\alpha$-ReLU as drop-in replacement for $\alpha$-entmax.

We test $\arelu$ at output in a neural machine translation task \cite{DBLP:conf/nips/SutskeverVL14}, which is essentially a conditional text generation task. 
Compared to open-ended text generation, there is a clearer metric of the quality of the generated text---the BLEU score \cite{DBLP:conf/acl/PapineniRWZ02}. As in open-ended text generation, at each prediction step, the NMT system needs to make a choice from all words (subwords) of the vocabulary, the size of which can reach several tens of thousands. Therefore, the sparsity of the output distribution becomes critical in such setups, since it can explicitly prevent the occurrence of most of the words that are inappropriate in the context. 


\subsection{Setup}\label{sec:setup}

\paragraph{Data.} We conduct experiments on three datasets of varied sizes:
\begin{itemize}
    \item IWSLT'14 De$\to$En \cite{cettolo2014report}, 172K training examples,
    \item WMT'14 En$\to$De \cite{DBLP:conf/wmt/BojarBFHKLMPPSS14}, 4.5M training examples,
    \item WMT'13 En$\to$Ru \cite{bojar-etal-2013-findings}, 1.3M tranining examples.\footnote{We did not use the Yandex 1M Parallel Corpus because of its license restrictions.}
\end{itemize}
We preprocess all datasets using the byte pair encoding algorithm \cite{DBLP:conf/acl/SennrichHB16a} with 10K merge operations on IWSLT, 40K merge operations on WMT En$\to$De, and 60K merge operations on WMT En$\to$Ru. We report detokenized case-sensitive BLEU with SacreBLEU \citep{post2018call}.\footnote{\textit{BLEU+case.mixed+lang.en-de+numrefs.1+smooth.exp+tok.13a+version.1.5.1}}

\paragraph{Hyperparameters $\alpha$ and $\tau$.} In all experiments we set $\alpha = 1.5$, because this value was recommended by \citet{DBLP:conf/acl/PetersNM19,DBLP:conf/naacl/PetersM21} as the middle ground between $\alpha = 1$ (softmax) and $\alpha = 2$ (sparsemax). 

The value for $\tau$ is chosen as follows: we run the first batch through a non-trained neural network, which has 1.5-entmax at the output, in the forward direction and determine the average $\tau$ value across the batch. This value is then used to train the 1.5-ReLU network. Our preliminary experiments have shown that 1.5-ReLU convergence is sensitive to the $\tau$ value, and that having output close to the probability distribution early in the learning phase works well with the rest of hyperparameters which are set to their default values.

\paragraph{Training.} We trained the Transformer Base \cite{DBLP:conf/nips/VaswaniSPUJGKP17} using the OpenNMT-py 2.0 toolkit \cite{klein-etal-2017-opennmt}. Optimization details are in Appendix~\ref{app:opt}.

\subsection{Results}
\label{sec:results} The results are given in  Table~\ref{tab:nmt_results}. Reported are test BLEU scores for best checkpoints which are selected based on validation BLEU. We observe
that the 1.5-ReLU performs on par with 1.5-entmax or better, while sparsemax is inferior to all others. 

\paragraph{Training Time.}

\begin{figure*}[htbp]
\begin{center}
\includegraphics[width=.32\textwidth]{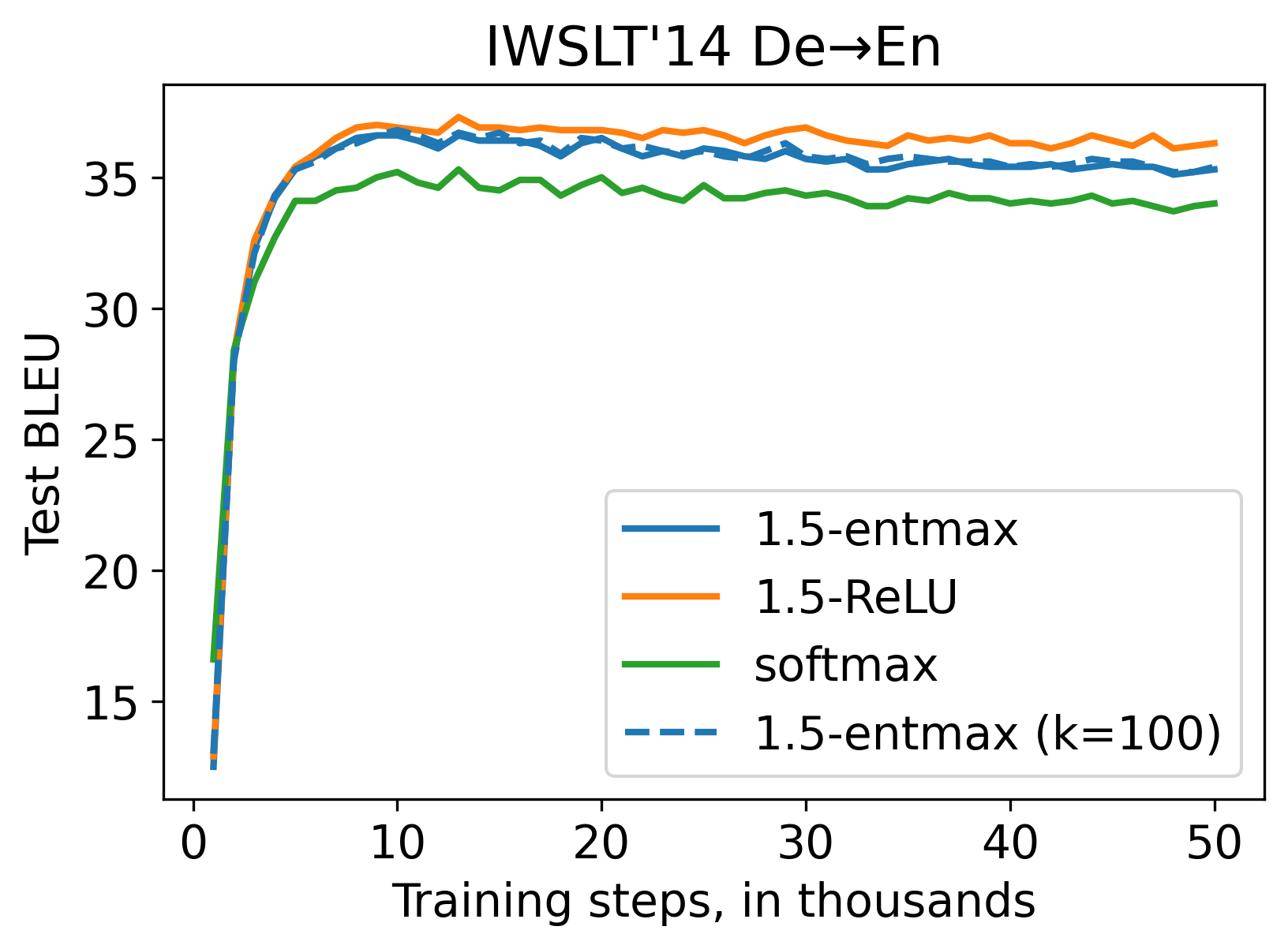}\hfill\includegraphics[width=.32\textwidth]{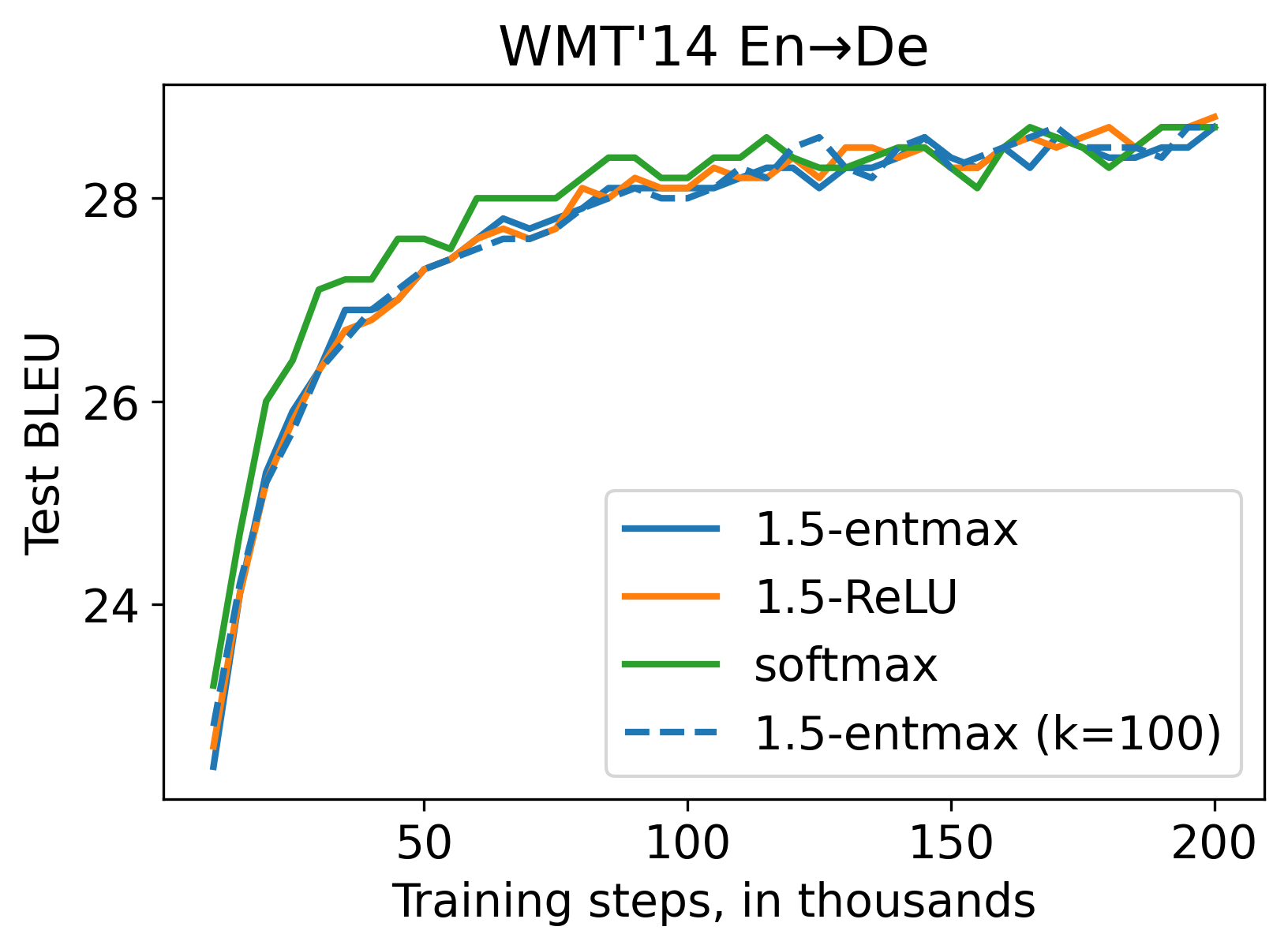}\hfill\includegraphics[width=.32\textwidth]{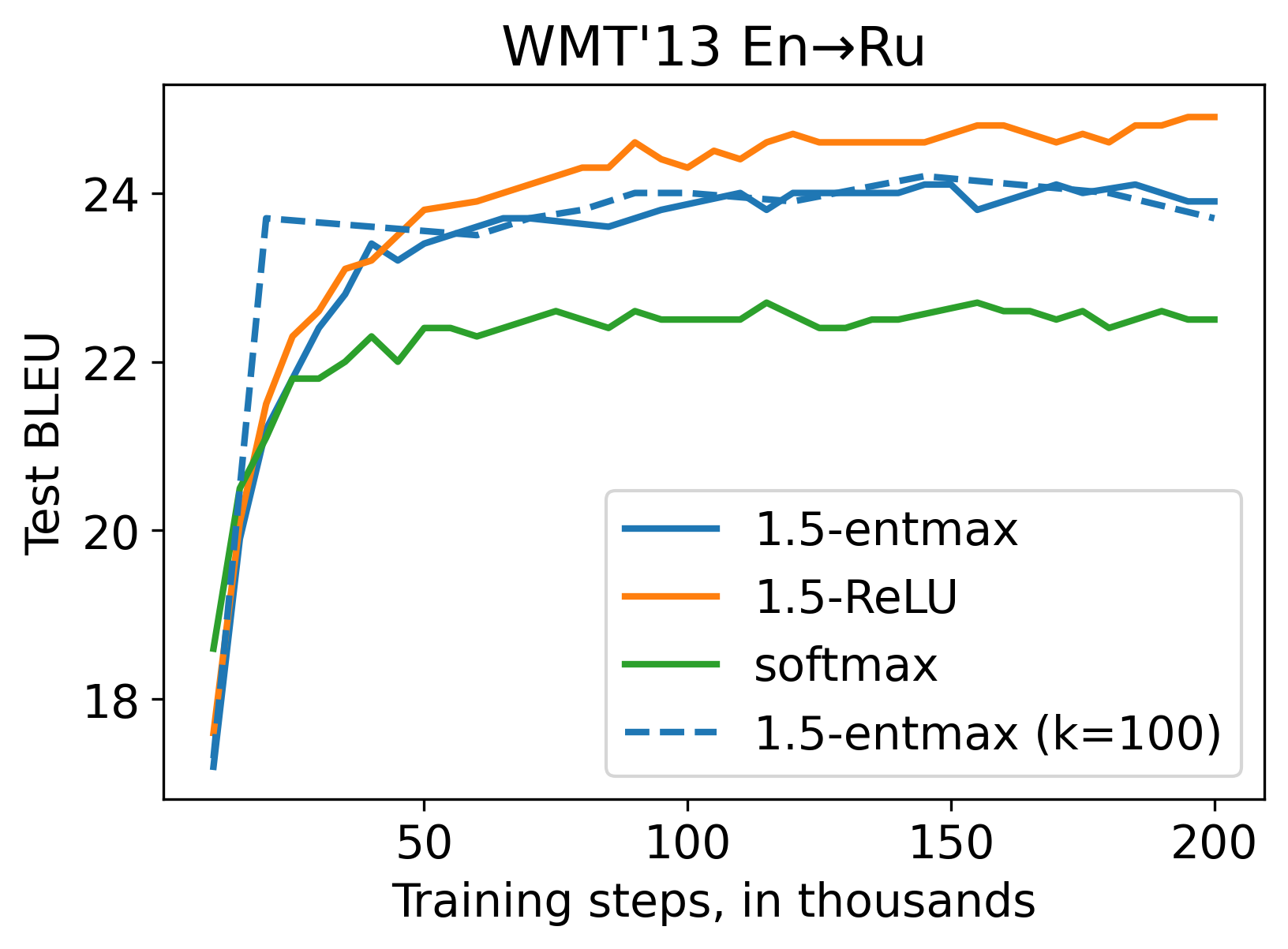}
\end{center}
\caption{Training dynamics in training steps.}
\label{fig:train_dyn_rel}
\end{figure*}

\begin{figure*}[htbp]
\begin{center}
\includegraphics[width=.32\textwidth]{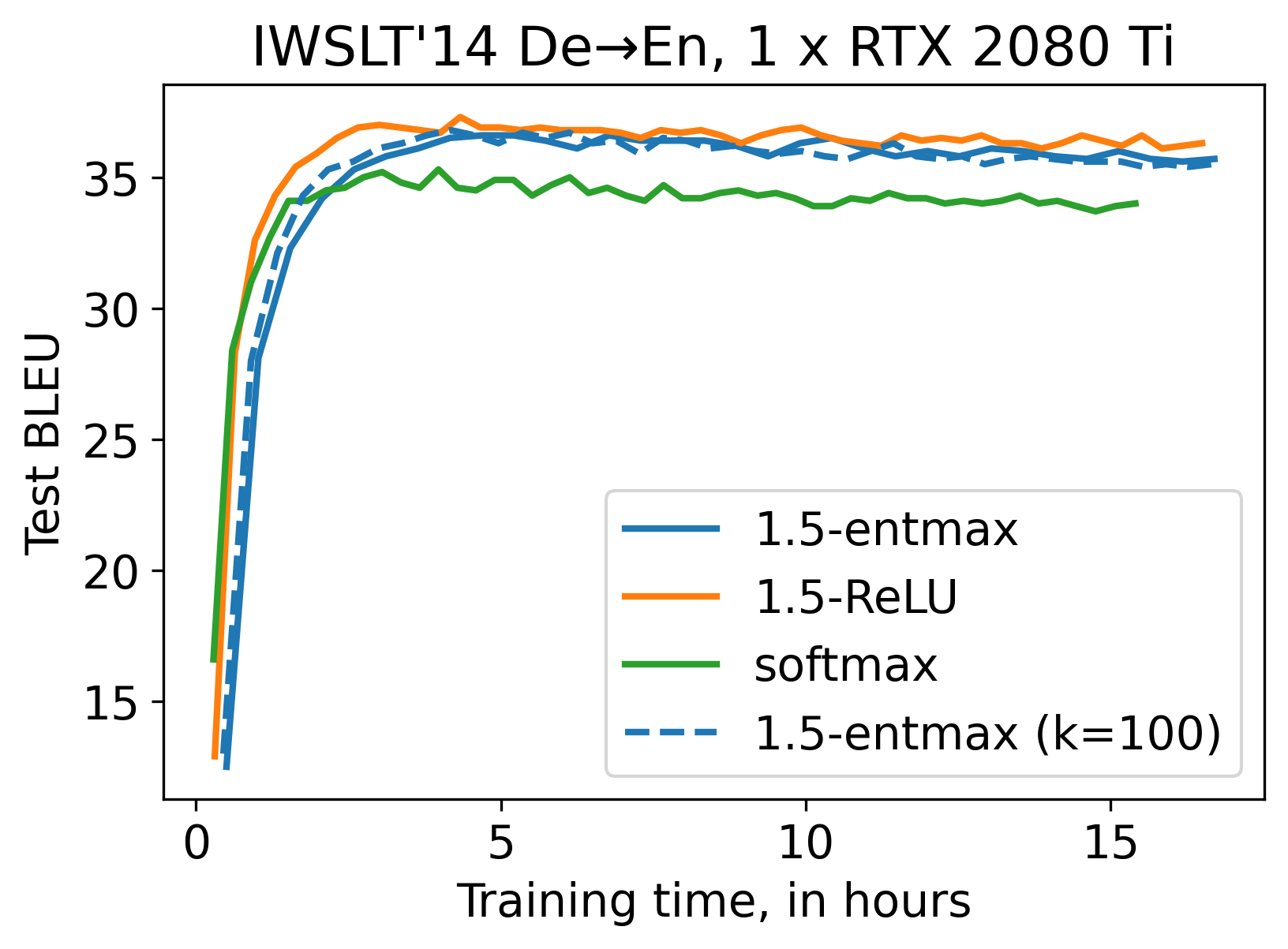}\hfill\includegraphics[width=.32\textwidth]{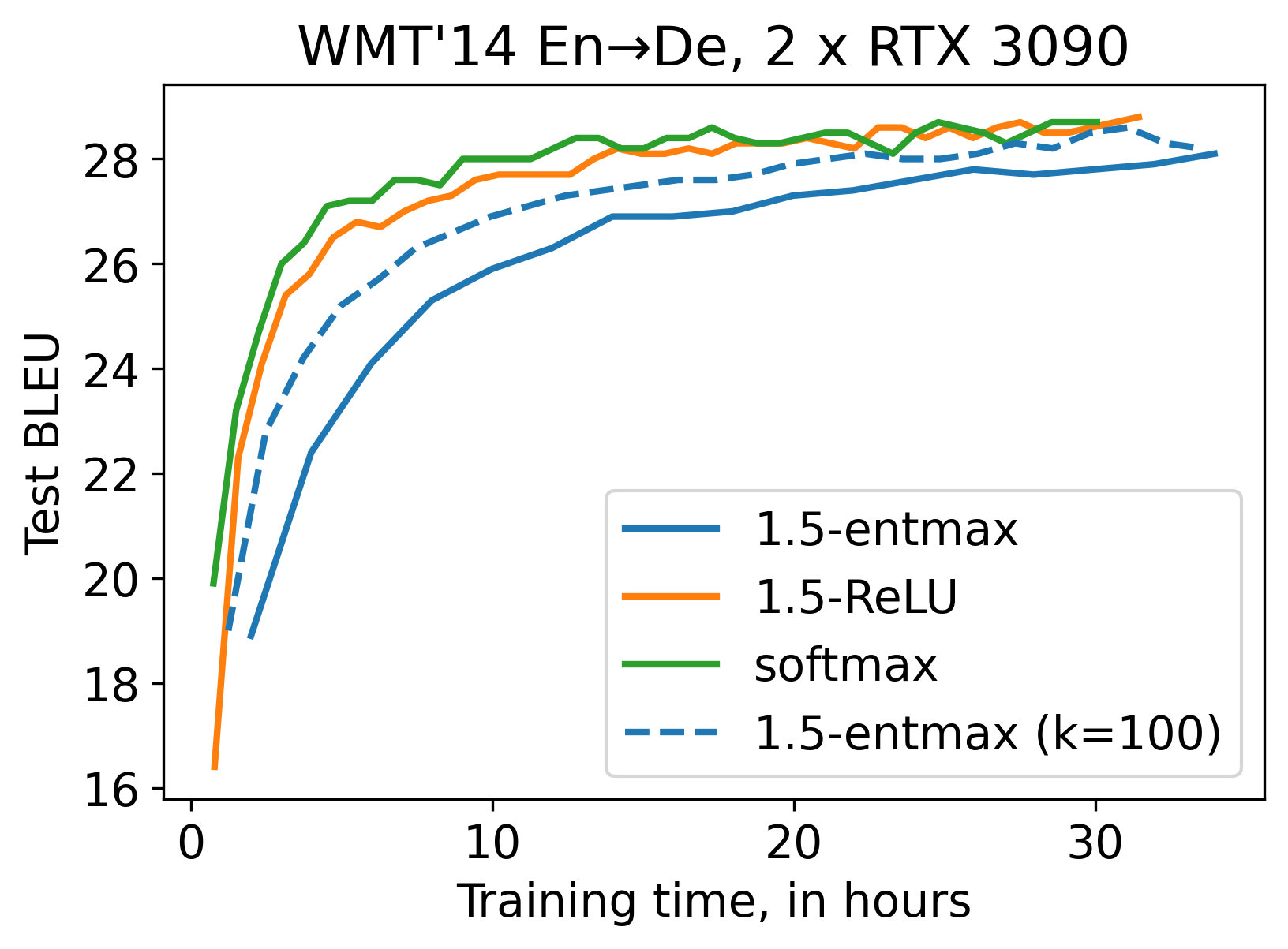}\hfill\includegraphics[width=.32\textwidth]{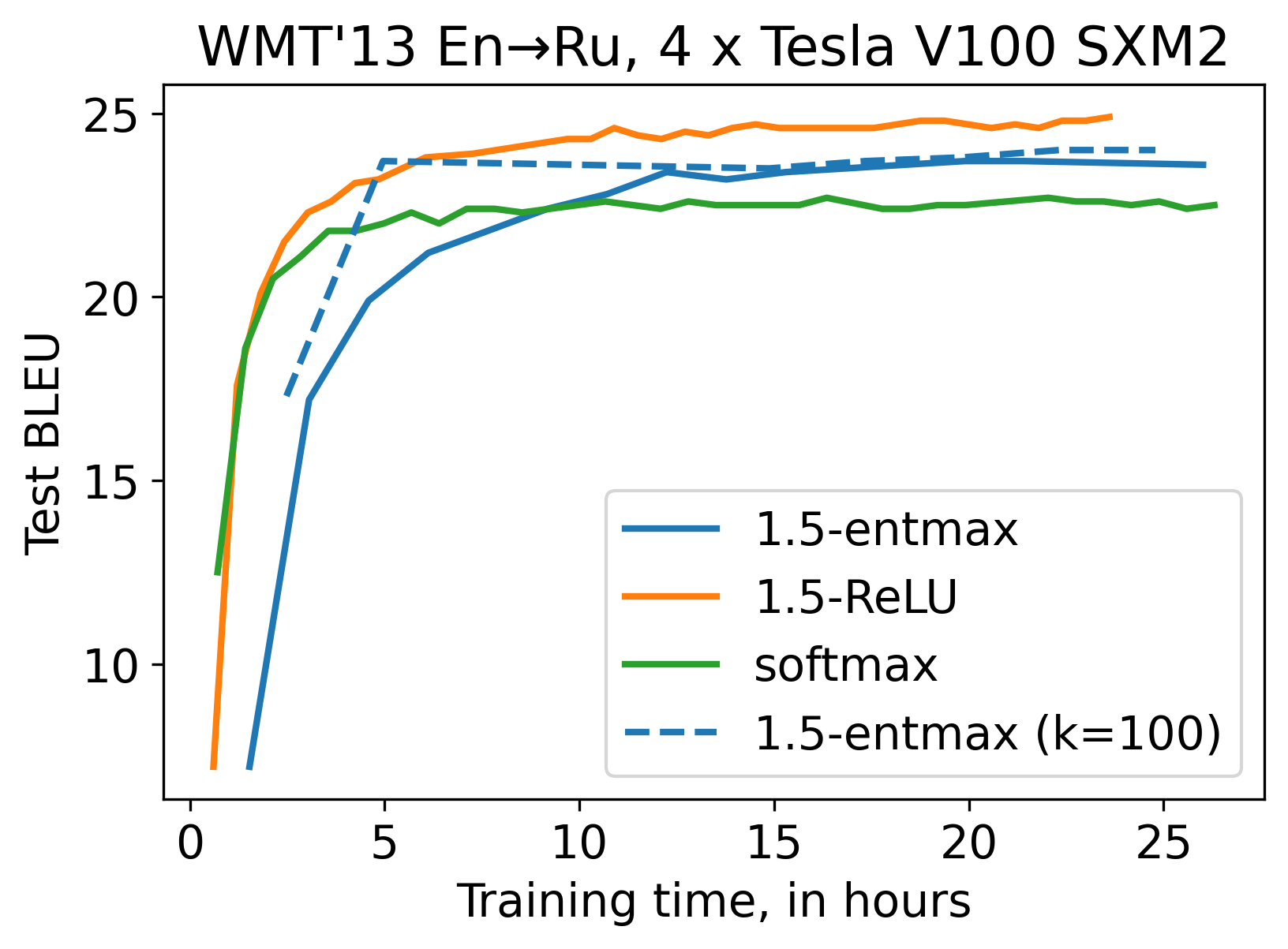}
\end{center}
\caption{Training dynamics in absolute time. \emph{1.5-entmax (k=100)} is a variant of 1.5-entmax in which sorting is performed only for the largest $k=100$ logits.} 
\label{fig:train_dyn_abs}
\end{figure*}

Fig.~\ref{fig:train_dyn_rel}\&\ref{fig:train_dyn_abs} show the training dynamics in training steps and in wall time on WMT'14 En$\to$De.  Despite the closeness of performance in intermediate steps and at the end of training, we see that on the larger datasets 1.5-entmax is slower in wall time  than softmax and 1.5-ReLU. 

To speed up the learning process, \citet{DBLP:conf/acl/PetersNM19} recommended limiting the number of sorted logits in the $\alpha$-entmax to the $k$ largest logits. We tried this  using $k = 100$, which is the default value in the author's implementation of $\alpha$-entmax.\footnote{\url{https://github.com/deep-spin/entmax}} The resulting training dynamics are shown as dashed curves in Fig.~\ref{fig:train_dyn_rel}\&\ref{fig:train_dyn_abs}. As we can see, partial sorting indeed speeds up the learning process, and at the same time does not harm the quality of the translation compared to $\alpha$-entmax with full sorting. But in the end, learning is still slower than in the case of 1.5-ReLU. Of course, one can try to select such $k$ that the speed of calculating the 1.5-entmax will be as close as possible to the speed of 1.5-ReLU without losing quality, but this requires additional efforts on the part of the user, and this must be done for each case separately. Also note that both 1.5-entmaxes (with full and partial sorting) cannot learn the English-Russian data set as well as 1.5-ReLU.

In this regard, 1.5-ReLU does not require additional fine-tuning, converges as fast as softmax in absolute time and performs on par or better. Thus 1.5-ReLU combines all three desired properties: computation speed, task performance, and sparsity of output.

\begin{figure}
    \centering
    \includegraphics[width=.45\textwidth]{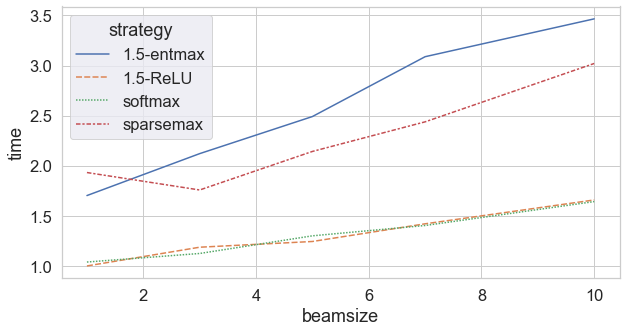}
    \caption{Normalized inference for WMT En$\to$Ru with different beam sizes.}
    \label{fig:infTime}
\end{figure}

\paragraph{Inference Time.} We measured inference time of translating the WMT En$\to$Ru test data with the different strategies and with different beam sizes.
The results---normalized by the smallest value---are shown in Fig.~\ref{fig:infTime}.
As can be seen the relative difference seems independent of the beam size: softmax is almost twice faster than 1.5-entmax (with full sorting over the logits). Even though the softmax version is optimized through the softmax CUDA kernel, it performs equivalent to the 1.5-ReLU model in terms of computation speed.

\section{Analysis}\label{sec:analysis}

\subsection{Empty Translations} 
We remind the reader that the \textit{cat got your tongue} problem \cite{DBLP:conf/emnlp/StahlbergB19} is one of the main motivations for using sparse transformations when generating text. As \citet{DBLP:conf/naacl/PetersM21} have shown, 1.5-entmax successfully tackles this problem by significantly lowering the proportion of cases where an empty string is more likely than the beam search hypothesis. For 1.5-ReLU, we also calculated this proportion, and compared it with the proportions for softmax and sparsemax (Table~\ref{tab:empty_perc}). As we see, 1.5-ReLU also successfully tackles the \textit{cat got your tongue} problem.

\begin{table}[htbp]
    \centering
    \begin{tabular}{l c c c c}
    \toprule
        Output & IWSLT  & WMT & WMT \\
        Transform & De$\to$En & En$\to$De & En$\to$Ru\\
        \midrule
        softmax & 7.5\% & 29.8\% & 31.7\%\\
        sparsemax & 0\% & 0.03\% & 0\%\\
        1.5-entmax & 0\% & 0.2\% & 0\%\\
        1.5-ReLU & 0\% & 0.3\% & 0.1\%\\
        \bottomrule
    \end{tabular}    
    \caption{Percentage of development set examples for which the model assigns higher probability to the empty string than to the beam-decoded hypothesis.}
    \label{tab:empty_perc}
\end{table}

\begin{figure*}
\begin{center}
\begin{minipage}[t]{.55\textwidth}
\begin{center}
\includegraphics[width=\textwidth]{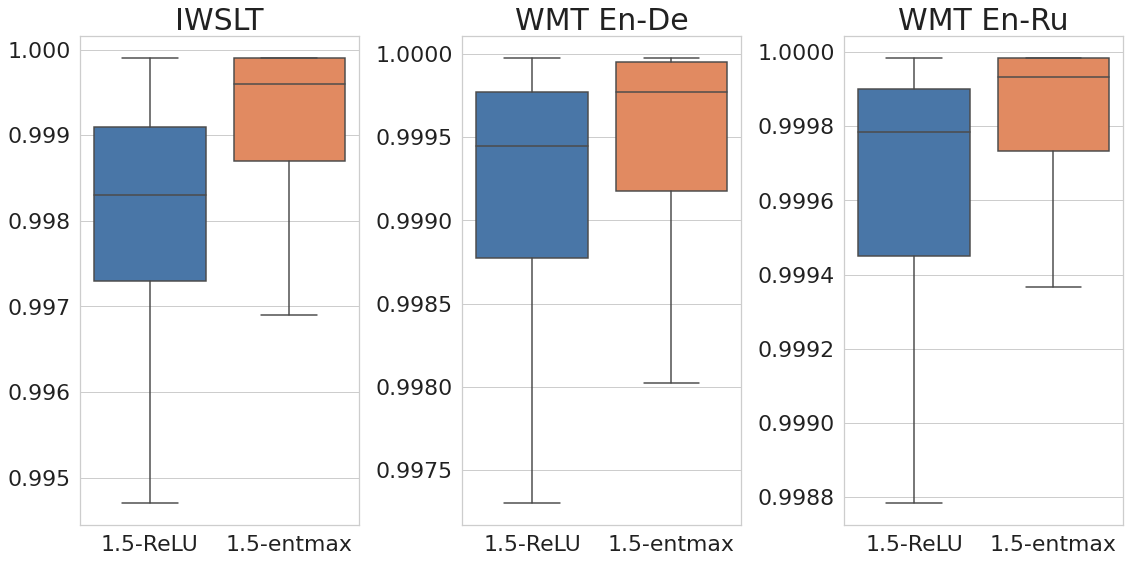}
\end{center}
\caption{Sparsity as proportion of zero components after applying 1.5-ReLU and  1.5-entmax, test sets.}
\label{fig:sparsity}
\end{minipage}\qquad\begin{minipage}[t]{.3\textwidth}
    \begin{center}
    \includegraphics[height=125pt]{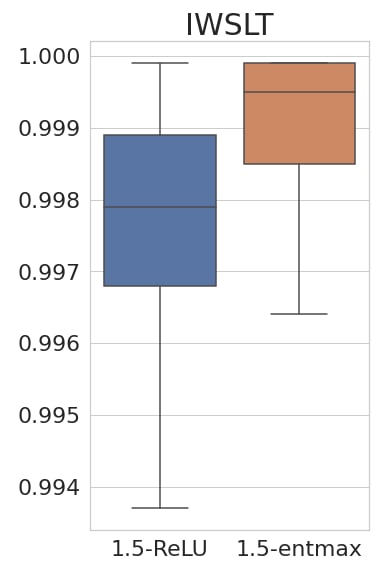}
    \end{center}
    \caption{Sparsity on training set.}
    \label{fig:sparsity_train}
\end{minipage}
\end{center}
\end{figure*}

\subsection{Sparsity}
To compare the sparsity of 1.5-ReLU and 1.5-entmax we depict in Fig.~\ref{fig:sparsity} the distributions of the number of zero components after applying these transformations (recall that for softmax all components are always nonzero). Since we constructed the $\alpha$-ReLU in such way that it mimics the $\alpha$-entmax (at least in the early stages of training), we expected that these two transformations would have similar properties, including sparsity. However, this is not the case: as we can see, the 1.5-ReLU is significantly less sparse than the 1.5-entmax. It is noteworthy that lower sparsity in this case correlates with a better performance in the translation task (see Table~\ref{tab:nmt_results}).  

A possible explanation for the difference in sparsity levels could be that $\alpha$-ReLU, in contrast to $\alpha$-entmax, behaves significantly differently on the test set than on the training set. However, this is not the case: for example, comparing the sparsity on the IWSLT training set (Fig.~\ref{fig:sparsity_train}), we see that the distributions of non-zero components are almost the same as on the test set for 1.5-ReLU and 1.5-entmax.

Note that the sparsity of $\alpha$-ReLU and $\alpha$-entmax is approximately the same at the beginning of training due to how we initialize $\tau$ in 1.5-ReLU (making it as close as possible to 1.5-entmax's $\tau$ in the untrained model, Sec.~\ref{sec:setup}). However, during training, $\arelu$'s $\tau$ remains \emph{fixed}, and the model can only adapt the logits themselves so that $\arelu(\mathbf{z})$ converges to the corresponding one-hot vector. At the same time, in $\entmax$, $\tau(\mathbf{z})$ adapts \emph{together} with logits $\mathbf{z}$. We hypothesize that during training, the entmax's $\tau(\mathbf{z})$ gradually increases which entails greater sparsity by the end of the training. However, the logits themselves also change during training, so the increase in $\tau$ may not be the cause of greater sparsity. To find out, we track the dynamics of mean logit norm $\|\mathbf{z}\|$ and mean $\tau$ during training for both 1.5-entmax and 1.5-ReLU (Fig.~\ref{fig:tau_z_evo}).
\begin{figure}[htbp]
    \centering
    \includegraphics[width=.48\textwidth]{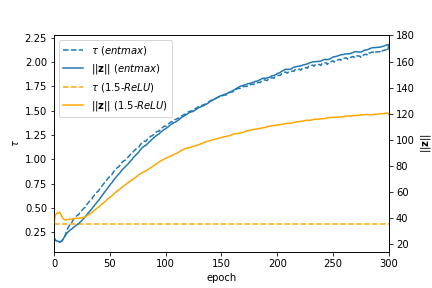}
    \caption{Evolution of the mean $\tau(\mathbf{z})$ and $\|\mathbf{z}\|$ during training for 1.5-entmax and 1.5-ReLU models on IWSLT'14 En$\to$De.}
    \label{fig:tau_z_evo}
\end{figure}
As we can see, the logit sizes grow in both cases. At the same time, the 1.5-entmax's $\tau(\mathbf{z})$ increases following the logit size, while the 1.5-ReLU's $\tau$ remains constant. From this we conclude that the sparsity of 1.5-entmax is inevitably less than the sparsity of 1.5-ReLU.

\subsection{Impact of $\tau$}
The selection of $\tau$ was described in Section~\ref{sec:setup}. However, the question arises: does the described approach lead to the choice of the \textit{optimal} $\tau$? To find out, we trained the $\alpha$-ReLU models for $\tau\in\{0, 0.1, 0.2, ..., 0.9, 1, 2, 5, 10\}$ on the IWSLT data. Note that all of these $\tau$'s have led to almost the same result at the end of the training (as predicted by Lemma~\ref{lem:gradient}). In Fig.~\ref{fig:tau}, we present the dynamics of early training only for $\tau\in\{0, 0.1, 0.2, 0.3, 5, 10\}$, since the curves for $\tau\in\{0.4, ..., 0.9, 1, 2\}$ practically coincided with the optimal curve corresponding to $\tau = 0.3$. 
\begin{figure}[htbp]
    \begin{center}
    \includegraphics[width=.47\textwidth]{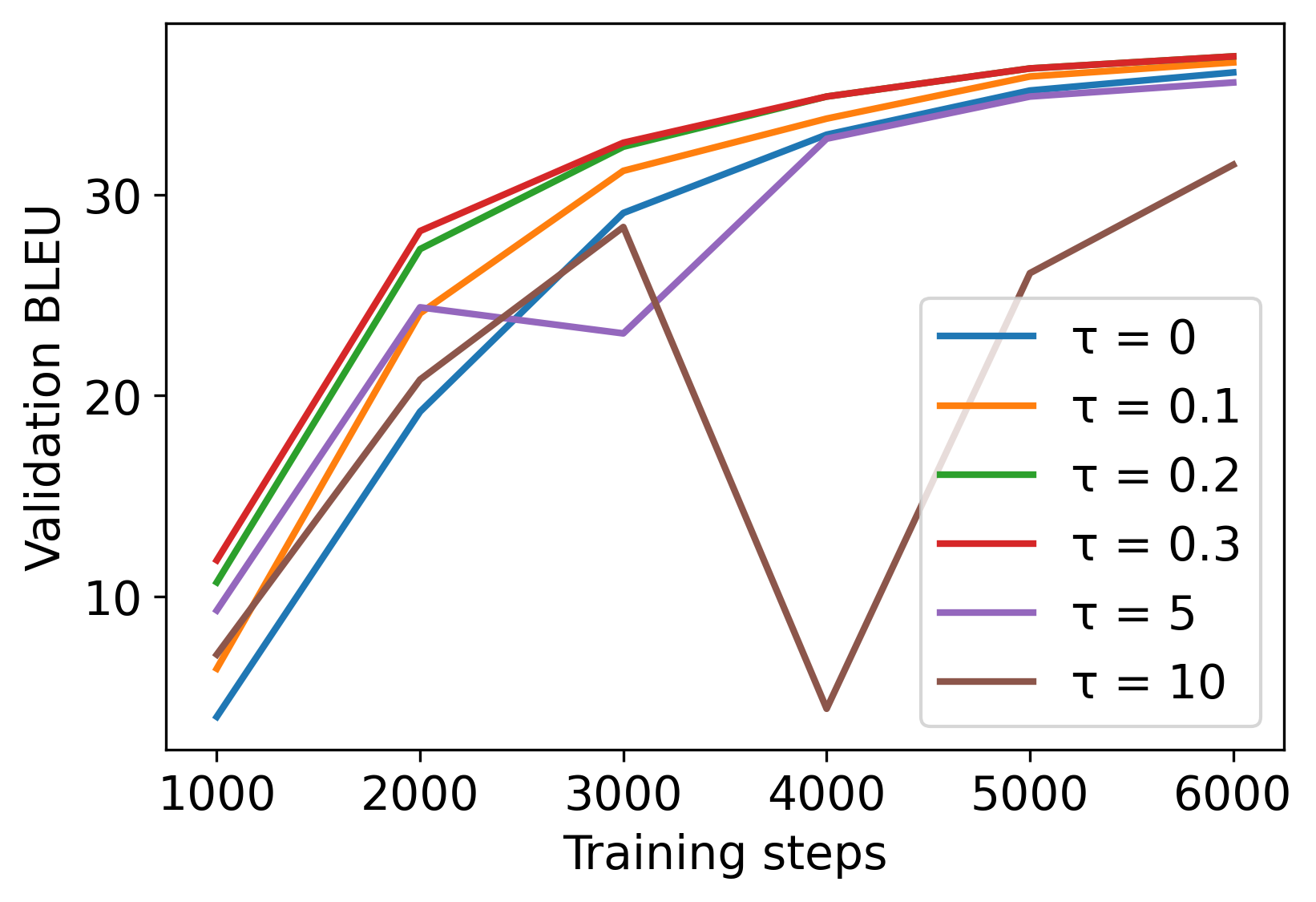}
    \end{center}
    \caption{Impact of $\tau$ on training dynamics, IWSLT'14 En$\to$De.}
    \label{fig:tau}
\end{figure}
Note that our $\tau$ selection method gave a value of $0.33$, thus we have no evidence against the adequacy of our method.

\subsection{Estimation of $\tau$ without data}
On closer inspection, we noticed that the pre-entmax logits in the \emph{untrained} Transformer model are distributed according to the normal law, regardless of what data is supplied to the input, Shapiro-Wilk test, p-value $> 0.15$. This allows us, using asymptotic theory, to estimate $\tau$ as
\begin{equation}
\hat{\tau}=\sqrt{\frac{d_{\text{model}}}{{2(d_{\text{model}}+d_{\text{vocab}})}}}\cdot \Phi^{-1}(1-p^\ast),\label{eq:tau_est}
\end{equation}
where $d_{\text{model}}$ is the size of hidden representations, $d_{\text{vocab}}$ is the  vocabulary size for a target language, $\Phi^{-1}(\cdot)$ is the probit function and $p^\ast$ is the solution of a non-linear equation that involves functions related to the standard normal distribution (see Appendix~\ref{app:tau} for details). Table \ref{tab:tau} compares the $\tilde{\tau}$ calculated by running data through an untrained model with the estimate $\hat\tau$ obtained from  \eqref{eq:tau_est}.
\begin{table}[htbp]
    \centering
    \begin{tabular}{l c c c}
    \toprule
    & IWSLT'14  & WMT'14  & WMT'13  \\
    & De$\to$En & En$\to$De & En$\to$Ru \\
    \midrule
    $d_{\text{model}}$ & 512 & 512 & 512 \\
    $d_{\text{vocab}}$ & 10,000 & 40,000 & 60,000\\
    $p^\ast$ & .0184 & .0171 & .0169\\
    \midrule
    $\tilde{\tau}$ & .33 & .17 & .14 \\
    $\hat{\tau}$ & .33 & .17 & .14 \\
    \bottomrule
    \end{tabular}
    \caption{Estimating threshold of 1.5-entmax: $\tilde{\tau}$ is a value obtained by running a data through an untrained model; $\hat\tau$ is an estimate based on asymptotic theory, i.e. without running the data through the model.}
    \label{tab:tau}
\end{table}
As we can see, $\hat\tau$ practically coincides with $\tilde{\tau}$ with an accuracy of two decimal places. Unfortunately, the formula \eqref{eq:tau_est} is not universal: it is only true for the  Transformer architecture.

\subsection{Self-normalization}
The attentive reader may have noticed that the output of $\alpha$-ReLU is not normalized, i.e. the components of $\arelu(\mathbf{z})$ do not have to sum up to 1. Accordingly, the question arises: how correct is it to compare translation scores at different steps of the beam-search decoding if the conditional probabilities are not normalized? However, the comparison is possible if the $\arelu(\mathbf{z})$ components add up to approximately the same number, i.e. if the model is self-normalizing. To check this, we ran the trained $\alpha$-ReLU model on the IWSLT and WMT'14 test sets, and looked at the distribution of $\sum_i\arelu_i(\mathbf{z})$ at each decoding step. The results are shown in Fig.~\ref{fig:self_norm}.
\begin{figure}[htbp]
    \centering
    \includegraphics[width=.45\textwidth]{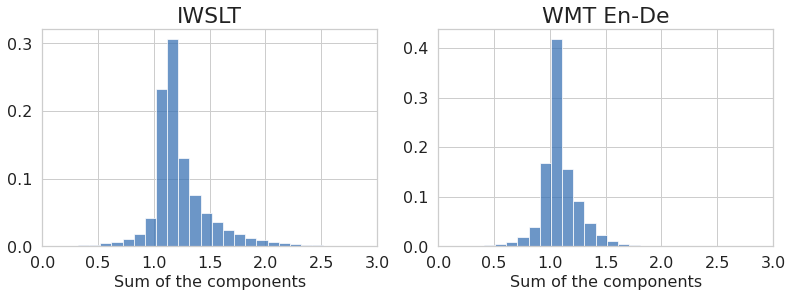}
    \caption{Distribution of the sum of  $\arelu(\mathbf{z})$ components across the IWSLT'14 and WMT'14 test sets: $\alpha$-ReLU self-normalizes.}
    \label{fig:self_norm}
\end{figure}
As we can see, the sum of the $\arelu(\mathbf{z})$ components  concentrates well around its mean $\approx1.24$ (IWSLT) and $1.09$ (WMT'14), which might indicate that the model indeed has a self-normalization property.

\subsection{Training Dynamics}
As we noted in Sect.~\ref{sec:results}, the training dynamics are similar in all three cases (softmax, 1.5-entmax, 1.5-ReLU) when time is measured in training steps. Here we attempt to explain this phenomenon through the recently proposed Neural Tangent Kernel (NTK) approach of \citet{DBLP:conf/nips/JacotHG18}. Roughly speaking, the NTK theory suggests that a sufficiently wide neural network trains like a kernel regression. We use this theory to show (in Appendix~\ref{app:dyn_proof}) that in all three cases the logits $\mathbf{z}(x,t)$ for a training instance $x$ at a training step $t$ evolve (approximately) according to the same differential equation
\begin{equation}
\frac{d\mathbf{z}}{dt}=-{\E}_{(x',y')}[\mathbf{K}_\sigma(x,x')\cdot(\sigma(\mathbf{z}')-\mathbf{e}_{y'})],\label{eq:mde}
\end{equation}
where expectation is over training examples $(x',y')$, $\sigma(\cdot)$ is one of the transformations considered (softmax, $\alpha$-entmax, or $\alpha$-ReLU), and $\mathbf{K}_\sigma(x,x')\in\mathbb{R}^{d\times d}$ is a positive semi-definite matrix that depends on $\sigma$. The Equation \eqref{eq:mde} is a non-linear matrix differential equation which in general cannot be solved analytically. However, it has an equilibrium point $\mathbf{z}(x,t)$ such that ${\E}_{(x',y')}[\mathbf{K}_\sigma(x,x')\cdot(\sigma(\mathbf{z}')-\mathbf{e}_{y'})]=\mathbf{0}$, thus its solution converges to this point as $t\to\infty$. This similarity in the evolution of $\sigma(\mathbf{z})$ implies the similarity in the evolution of the perfomance metric---such as BLEU---across all three transformations.

\subsection{Human Evaluation}

Although the BLEU metric \cite{DBLP:conf/acl/PapineniRWZ02} has stood the test of time, it is still an automated assessment of translation quality. To double-check the reliability of the results from Table~\ref{tab:nmt_results}, we decided to manually evaluate the translations from the WMT'13 En$\to$Ru test split. To do this, we followed the human evaluation setup from \cite{berard-etal-2019-machine}. We formed two random samples of 135 instances each and gave them to two annotators. 45 instances were shared across two samples in order to calculate inter-annotator agreement. Each instance consists of an original sentence in English and 4 candidate translations into Russian (reference, softmax, entmax, $\alpha$-ReLU). The annotators were to rate each translation on a 4-point scale. For annotation instructions, see Appendix~\ref{app:annot}.

The order of candidate translations was shuffled for each instance, so the annotators did not know which sentence is from which model. Nevertheless, the annotator always had a good chance of guessing which translation was the reference one, due to the large difference in quality between human and machine translation.

The results of human evaluation are shown in Table~\ref{tab:human}.
\begin{table}[]
    \centering
    \begin{tabular}{l c c}
        \toprule
        Model & Avg. Score & Std. Dev.\\
        \midrule
        Reference & 3.9 & 0.30 \\
        Softmax & 3.3 & 0.75\\
        1.5-entmax & 3.2 & 0.74\\
        1.5-ReLU & 3.3 & 0.74 \\
        \bottomrule
    \end{tabular}
    \caption{Results of Human Evaluation across 270 random examples (with repetitions) from WMT'13 En$\to$Ru test split. Scores are on a 4-point scale.}
    \label{tab:human}
\end{table}
Cohen's $\kappa = 0.56$, indicating moderate agreement between annotators. As we can see, all three models give approximately the same translation quality, and all three are significantly inferior to the reference translation. This is generally consistent with the results of 1.5-ReLU and 1.5-entmax in Table~\ref{tab:nmt_results}, but at the same time casts doubt on the softmax lag behind 1.5-ReLU and 1.5-entmax as the BLEU metric suggests.

In Appendix~\ref{app:trans_examples} we give a few examples where 1.5-ReLU translates better than 1.5-entmax and vice versa.

\section{Related Work}
\paragraph{Sparse seq2seq models.} Our proposed $\alpha$-ReLU transformation is based on the $\alpha$-entmax transformation of \citet{DBLP:conf/acl/PetersNM19}, which in turn is a generalization of the sparsemax transformation \cite{DBLP:conf/icml/MartinsA16}. In our work, we study sparseness at the output of a neural network. Nevertheless, there are a number of works aimed at sparsification within a neural network. For example, \citet{DBLP:conf/acl/MalaviyaFM18,DBLP:conf/acl/PetersNM19,DBLP:conf/emnlp/CorreiaNM19} show that sparsemax and $\alpha$-entmax can replace softmax in the attention mechanism with some success. A recent work of \citet{DBLP:conf/emnlp/ZhangTS21} attempted to replace softmax with a component-wise ReLU in the attention mechanism. Unfortunately, in its pure form, this replacement leads to the inability of the model to learn at all, since its loss function does not decrease during optimization. The authors solve this problem by adding a normalizing layer on top of the attention layer.

These and other works \cite{DBLP:journals/taslp/ZhangZLZ19} state that sparsity in the weights of attention produces more interpretable patterns. However, \citet{DBLP:conf/acl/MeisterLAC20} questioned this claim and were unable to find clear evidence to support it. Therefore, in this work, we focused on the application of $\alpha$-ReLU to the output of the transformer model, and not to the mechanism of attention, but at the same time we do not deny the possibility of studying the latter.

\paragraph{Self-normalization.} Self-normalizing training aims to bypass the need of normalization during inference time.
This is done by tweaking the learning mechanism so that the sum of all predictions sums (approximately) to a constant value.
Theoretical work on why this works is poorly understood~\citep{andreas2015accuracy} but early work in neural machine translation has shown its empirical value.
\citet{vaswani-etal-2013-decoding} achieves that by using noise-contrastive estimation (the neural model is used to re-rank the output of a hierarchical phrase-based machine translation system).
Noise-contrastive estimation is also the standard training mechanism for word2vec (more popular than the alternative hierarchical softmax), which also eschews any expensive normalization.
Differently, \citet{devlin-etal-2014-fast} changes the training loss to include a factor that encourages the normalizing factor to be 1.
At inference time, this is just assumed and decoding time is reported to achieve a 15x speed-up.

\section{Limitations and Risks}
We believe that the main limitations of our work are as follows:
\begin{itemize}
    \item $\alpha$-ReLU's output is still not a probability distribution, as required by the classical formulation of a probabilistic classification model.
    \item $\tau$ evaluation requires either running the data through an untrained model with $\alpha$-entmax at the output, or deriving a formula similar to \eqref{eq:tau_est} for each individual architecture.
    \item Our approach only works for the case when $\alpha$-ReLU is used at the output of the model, but it is not clear how to use it as an alternative to softmax/$\alpha$-entmax in the attention layer.
\end{itemize}
The last mentioned limitation leads to the potential risk of inability to learn if $\alpha$-ReLU is misused in the intermediate layers of the neural network such as attention layers. The experiments of \citet{DBLP:conf/emnlp/ZhangTS21} using vanilla ReLU (2-ReLU with $\tau = 0$ in our notation) instead of softmax to produce attention weights lead to a divergence of the loss function of the Transformer model. This translates into a waste of energy, especially when training large models on large datasets. Therefore, we believe that in the future, a preliminary mathematical analysis and/or experiments with small models on small datasets should be carried out as to why the unnormalized distribution of attention weights leads to the inability of the model to learn.

\section{Conclusion}
It seems that the sparsity of the output is natural for (sub)word prediction models. Nevertheless, sparsity does not have to come with slowdown of computations, as our work shows. The proposed transformation, $\alpha$-ReLU, gives a sparse output, shows competitive performance, and is as fast as softmax. 
The reduced dependency on the vocabulary size seems particularly important in translation, where neural models are moving more and more towards multi-lingual ones, which in general have a much higher vocabulary size in order to accommodate enough tokens for all languages.

A natural extension of this work will be the evaluation of $\alpha$-ReLU in the problem of open-ended text generation, as well as a replacement for softmax in the attention layers of Transformer models.

Our standalone implementation of $\alpha$-ReLU in PyTorch is available at \url{https://github.com/MaxatTezekbayev/alpha-relu}.

\section*{Acknowledgements}
The authors would like to thank Christopher Dance, Andr\'e Martins, Bill Byrne, Vlad Niculae and anonymous reviewers for their feedback. The work of Maxat Tezekbayev  and Zhenisbek Assylbekov is supported by the Nazarbayev University Collaborative Research Program 091019CRP2109.


\bibliography{ref}
\bibliographystyle{acl_natbib}

\onecolumn

\clearpage
\appendix

\section{Optimization}\label{app:opt}

\setlist{noitemsep}

\paragraph{IWSLT'14 De$\to$En}
\begin{itemize}
    \item Architecture: Transformer, embedding size 512, 6 layers, 8 heads, hidden size 1024, shared vocabulary.
    \item Batch size: 4096 tokens (with gradient accumulation for 8 steps).
    \item Optimizer: \textit{ADAM}, $\beta_1=0.9$, $\beta_2=0.998$, noam decay, learning rate $2.0$, 4000 warmup steps.
    \item Dropout: $0.3$
    \item No label smoothing.
\end{itemize}

\paragraph{WMT'14 En$\to$De}
\begin{itemize}
    \item Architecture: Transformer, embedding size 512, 6 layers, 8 heads, hidden size 2048,  shared vocabulary of 40K tokens, shared embeddings and decoder embeddings.
    \item Batch size: 4096 tokens (with gradient accumulation for 4 steps).
    \item Optimizer: \textit{ADAM}, $\beta_1=0.9$, $\beta_2=0.998$, noam decay, learning rate $2.0$, 8000 warmup steps, average decay 0.0005.
    \item Dropout: $0.1$.
    \item Attention dropout: $0.1$. 
    \item No label smoothing.
\end{itemize}

\paragraph{WMT'13 En$\to$Ru}
Same as in WMT'14 En$\to$De, except that Dropout is 0.3.

\subsection{GPU Power Consumption}
\begin{table}[htbp]
    \centering
    \begin{tabular}{l  c c c}
    \toprule
    Dataset & IWSLT'14 En$\to$De & WMT'14 De$\to$En & WMT'13 En$\to$Ru \\
    GPU(s) & 1 $\times$ RTX 2080 Ti & 2 $\times$ RTX 3090 & 4 $\times$ Tesla V100 SXM2 \\
    Power consumption, W & 250 & 2$\times$320 & 4$\times$300 \\
    \midrule
    & \multicolumn{3}{c}{Training time, hours} \\
    \midrule
    softmax & 15.41 & 30.06 & 28.43\\
    sparsemax & 24.43 & 73.33 & 58.82\\
    1.5-entmax & 26.11 & 79.89 & 61.20\\
    1.5-ReLU & 16.50 & 31.44 & 24.21\\
    \midrule
    TOTAL hours & 82.44 & 214.72 & 172.67\\
    TOTAL kW-hours & 20.61 & 137.42 & 207.20\\ 
    \midrule
    \textbf{GRAND TOTAL kW-hours} & \multicolumn{3}{c}{\textbf{365.23}} \\
    \bottomrule
    \end{tabular}
    \caption{Power consumed by GPUs for training.}
    \label{tab:my_label}
\end{table}

We do not report $CO_2$ consumption, as experiments were run in different countries, making aggregate statistics difficult to compute.
The largest experiment (on WMT'13), were run in France, which benefits from a very low $CO_2$ emission intensity in its electrical mix.

\section{Proofs}

\paragraph{Notation.} We let $\mathbb{R}$ denote the  real numbers. Bold-faced lowercase letters ($\mathbf{x}$) denote vectors in Euclidean space, bold-faced uppercase letters ($\mathbf{A}$) denote matrices, plain-faced lowercase letters ($x$) denote scalars, $\|\cdot\|$ denotes the Euclidean norm: $\|\mathbf{x}\|:=\sqrt{\mathbf{x}^\top\mathbf{x}}$. The gradient of $f:\mathbb{R}^d\to\mathbb{R}$ is denoted by $\nabla f$. The Jacobian of $\mathbf{z}\mapsto g(\mathbf{z})$ is denoted by $\mathbf{J}_g(\mathbf{z})$. 
Also, we denote $\relu(x):=[x]_+:=\max\{x,0\}$, $[d]:=\{1,\ldots,d\}$,  $\Delta^{d-1}:=\{\mathbf{p}\in\mathbb{R}^d\mid\sum_i p_i=1,\,p_i\ge0\}$, $\mathbf{e}_y:=(0,\ldots,0,1,0,\ldots,0)$ where $1$ is at $y^\text{th}$ position. 

\subsection{Proof of Lemma~\ref{lem:gradient}.} \label{app:loss_proof}

First, let us calculate the Jacobian of the mapping $\mathbf{z}\mapsto\arelu(\mathbf{z})$. Recall that
$$
\arelu_i(\mathbf{z}):=[(\alpha-1)z_i-\tau]_+^{\frac{1}{\alpha-1}}.
$$
Therefore, the partial derivatives are given by
\begin{align*}
\frac{\partial[\arelu_i(\mathbf{z})]}{\partial z_i}&=\frac{1}{\alpha-1}\cdot[(\alpha-1)z_i-\tau]_+^{\frac{1}{\alpha-1}-1}\cdot(\alpha-1)=[(\alpha-1)z_i-\tau]_+^{\frac{2-\alpha}{\alpha-1}}\\
&=[\arelu_i(\mathbf{z})]^{2-\alpha},\\
\frac{\partial[\arelu_i(\mathbf{z})]}{\partial z_j}&=0.\qquad i\ne j
\end{align*}
Thus, the Jacobian can be written concisely as
\begin{equation}
    \mathbf{J}_{\arelu}(\mathbf{z})=\diag\{[\arelu(\mathbf{z})]^{2-\alpha}\},\label{eq:jacobian}
\end{equation}
where raising to power is done component-wise (i.e. $\mathbf{x}^\beta=[x_1^\beta,\ldots,x_d^\beta]$), and $\diag[\mathbf{x}]$ is a diagonal matrix with $\mathbf{x}$ on its diagonal.

Recall the definition of the Tsallis $\alpha$-entropy:
$$
\h_\alpha[\mathbf{p}]:=\frac{1}{\alpha(\alpha-1)}\left(1-\sum_j p_j^\alpha\right).
$$
Its gradient w.r.t. $\mathbf{p}$ is
$$
\nabla_\mathbf{p}\h_\alpha[\mathbf{p}]=-\frac{1}{\alpha-1}\mathbf{p}^{\alpha-1},
$$
Combining this with \eqref{eq:jacobian}, and using the chain rule, we have
\begin{align}
\nabla_\mathbf{z}\h_\alpha[\arelu(\mathbf{z})]&=\left[\mathbf{J}_{\arelu}(\mathbf{z})\right]^\top\cdot\left(-\frac{1}{\alpha-1}[\arelu(\mathbf{z})]^{\alpha-1}\right)\notag\\
&=\left[\diag\{[\arelu(\mathbf{z})]^{2-\alpha}\}\right]^\top\cdot\left(-\frac{1}{\alpha-1}[\arelu(\mathbf{z})]^{\alpha-1}\right)\notag\\
&=-\frac{1}{\alpha-1}[\arelu(\mathbf{z})]^{2-\alpha}\odot[\arelu(\mathbf{z})]^{\alpha-1}\notag\\
&=-\frac{1}{\alpha-1}\arelu(\mathbf{z}),
\label{eq:grad_tsallis}
\end{align}
where $\odot$ is the Hadamard product (element-wise multiplication), and we used $\diag[\mathbf{x}]\cdot\mathbf{y}=\mathbf{x}\odot\mathbf{y}$. Taking into account  \eqref{eq:grad_tsallis}, the gradient of the $\alpha$-ReLU loss \eqref{eq:sp_loss} w.r.t. $\mathbf{z}$ is
\begin{align}
\textstyle\nabla_{\mathbf{z}}&\ell(\mathbf{z},y)=\nabla_\mathbf{z}\left[(\arelu(\mathbf{z})-\mathbf{e}_y)^\top\left(\mathbf{z}-\frac{\tau}{\alpha-1}\mathbf{1}\right)\right]+\nabla_\mathbf{z}\h_\alpha[\arelu(\mathbf{z})]\notag\\
&=(\arelu(\mathbf{z})-\mathbf{e}_y)+\mathbf{J}^\top_{\arelu(\mathbf{z})}\left(\mathbf{z}-\frac{\tau}{\alpha-1}\mathbf{1}\right)-\frac{1}{\alpha-1}\arelu(\mathbf{z})\notag\\
&=(\arelu(\mathbf{z})-\mathbf{e}_y)+\frac{1}{\alpha-1}\left[\diag\{[\arelu(\mathbf{z})]^{2-\alpha}\}\right]^\top[(\alpha-1)\mathbf{z}-\tau\mathbf{1}]-\frac{1}{\alpha-1}\arelu(\mathbf{z})\notag\\
&=(\arelu(\mathbf{z})-\mathbf{e}_y)+\frac{1}{\alpha-1}\underbrace{[(\alpha-1)\mathbf{z}-\tau\mathbf{1}]_+^{\frac{2-\alpha}{\alpha-1}}\odot[(\alpha-1)\mathbf{z}-\tau\mathbf{1}]}_{\arelu(\mathbf{z})}-\frac{1}{\alpha-1}\arelu(\mathbf{z})\notag\\
&=\arelu(\mathbf{z})-\mathbf{e}_y,\notag
\end{align}
where in the fourth line we used $[\mathbf{x}]_+^\beta\odot\mathbf{x}=[\mathbf{x}]_+^\beta\odot[\mathbf{x}]_+=[\mathbf{x}]_+^{\beta+1}$. This concludes the proof.

\subsection{Approximation of $\tau$ for 1.5-entmax}\label{app:tau}

We derive the formula \eqref{eq:tau_est} in two steps: first in Lemma~\ref{lem:tau_normal}, we approximate $\tau(\mathbf{z})$ of 1.5-entmax when $\mathbf{z}$ is an arbitrary random sample from the normal distribution with zero mean and variance $\sigma^2$; next in Lemma~\ref{lem:var_transformer}, we compute $\sigma^2$ for the case when $\mathbf{z}$ is the pre-softmax vector of logits in the Transformer model.

\begin{lemma} \label{lem:tau_normal} Let $z_1,\ldots,z_d$ be independent and identically distributed random variables from the normal distribution $\mathcal{N}(0,\sigma^2)$. Then the thresholding function of $1.5$-$\mathrm{entmax}(\mathbf{z})$ can be approximated as
$$
\tau(\mathbf{z})\approx\frac{\sigma}{2}\Phi^{-1}(1-p^\ast),
$$
where $\Phi^{-1}(\cdot)$ is a probit function, and $p^\ast$ is the solution of
$$
\Phi^{-1}(1-p)= m(p)-\sqrt{\frac{4}{\sigma^2}\cdot\frac{\epsilon}{p}-s(p)}
$$
with
\begin{align}
m(p)&:=\frac{1}{p-\epsilon}\left[\phi(\Phi^{-1}(x)\right]_{x=\epsilon}^{x=p}\label{eq:m}\\
s(p)&:=\frac{1}{p-\epsilon}\left[-\phi(\Phi^{-1}(x))\cdot\Phi^{-1}(x)+x\right]_{x=\epsilon}^{x=p}-[m(p)]^2\label{eq:s}\\
\phi(t)&:=\frac{1}{\sqrt{2\pi}}e^{-\frac{t^2}{2}}\notag\\
\epsilon&:=\frac1d\notag
\end{align}
\end{lemma}
\begin{proof}
Let $z_{(1)}\ge \ldots \ge z_{(d)}$ be a sorting of $z_1,\ldots,z_d$ in  descending order. \citet{DBLP:conf/acl/PetersNM19} showed that
\begin{equation}
\tau(\mathbf{z})=\frac{M(k)}2-\sqrt{\frac{1}{k}-\frac{S(k)}4},\label{eq:tau_peters}
\end{equation}
where $k\in[d]$ is any index that satisfies
\begin{equation}
    \frac{z_{(k)}}2\ge \frac{M(k)}2-\sqrt{\frac{1}{k}-\frac{S(k)}4}\ge\frac{z_{(k+1)}}{2}\quad\Leftrightarrow\quad z_{(k)}\ge M(k)-\sqrt{\frac{4}{k}-S(k)}\label{eq:cond_k}\ge z_{(k+1)}
\end{equation}
with
$$
    M(k):=\frac1k\sum_{i=1}^k z_{(i)},\qquad S(k):=\frac1k\sum_{i=1}^k z_{(i)}^2-[M(k)]^2.
$$
Approximating $z_{(i)}$ by its asymptotic mean $\sigma\Phi^{-1}\left(1-\frac{i}d\right)$ \cite{10.5555/1373327}, and denoting $\epsilon:=\frac1d$, $p:=\frac{k}d$, we have
\begin{align*}
    M(k)&\approx\frac1k\sum_{i=1}^k\sigma\Phi^{-1}\left(1-\frac{i}d\right)\approx\frac{\sigma}{p-\epsilon}\int_\epsilon^p\Phi^{-1}(1-x)dx=\frac{\sigma}{p-\epsilon}\int_\epsilon^p-\Phi^{-1}(x)dx\\
    &=\frac{\sigma}{p-\epsilon}\left[\phi(\Phi^{-1}(x))\right]_{x=\epsilon}^{x=p}=\sigma m(p),
\end{align*}
where we approximated the average of finitely many numbers $\{\Phi^{-1}(1-i/d)\}_{i=1}^k$ by the mean value of the function $\Phi^{-1}(1-x)$, and then we used the fact that $-\phi(\Phi^{-1}(x))$ is an antiderivative for the probit function $\Phi^{-1}(x)$; and $m(p)$ is defined by \eqref{eq:m}.

Similarly, for the second empirical moment, we have
\begin{align*}
    \frac1k\sum_{i=1}^k z_i^2&\approx\frac1k\sum_{i=1}^k\left[\sigma\Phi^{-1}\left(1-\frac{i}d\right)\right]^2\approx\frac{\sigma^2}{p-\epsilon}\int_{\epsilon}^p[\Phi^{-1}(1-x)]^2dx=\frac{\sigma^2}{p-\epsilon}\int_{\epsilon}^p[\Phi^{-1}(x)]^2dx\\
    &=\frac{\sigma^2}{p-\epsilon}\left[-\phi(\Phi^{-1}(x))\cdot\Phi^{-1}(x)+x\right]_{x=\epsilon}^{x=p},
\end{align*}
and thus
$$
S(k)\approx\frac{\sigma^2}{p-\epsilon}\left[-\phi(\Phi^{-1}(x))\cdot\Phi^{-1}(x)+x\right]_{x=\epsilon}^{x=p}-[m(p)]^2=\sigma^2 s(p),
$$
where $s(p)$ is defined by \eqref{eq:s}. Hence, finding $k\in[d]$ that satisfies \eqref{eq:cond_k} is (approximately) equivalent to finding $p\in(0,1)$ that satisfies
\begin{equation}
\sigma\Phi^{-1}(1-p)= \sigma m(p)-\sqrt{4\cdot\frac{\epsilon}p-\sigma^2 s(p)}\quad\Leftrightarrow\quad\Phi^{-1}(1-p)= m(p)-\sqrt{\frac{4}{\sigma^2}\cdot\frac{\epsilon}p-s(p)}.\label{eq:p_eq}
\end{equation}
Let $p^\ast$ be the solution of \eqref{eq:p_eq}. Then, taking into account \eqref{eq:tau_peters}, we have
$$
\tau(\mathbf{z})\approx\frac{\sigma m(p^\ast)}{2}-\sqrt{\frac{\epsilon}{p^\ast}-\frac{\sigma^2 s(p^\ast)}{4}}=\frac{\sigma}{2}\left(m(p^\ast)-\sqrt{\frac{4}{\sigma^2}\cdot\frac{\epsilon}{p^\ast}-s(p^\ast)}\right)=\frac{\sigma}{2}\Phi^{-1}(1-p^\ast),
$$
which concludes the proof.
\end{proof}

\begin{lemma} \label{lem:var_transformer}
Let $\mathbf{z}=\mathbf{Wx}$ be a pre-softmax vector of logits in the OpenNMT-py \cite{klein-etal-2017-opennmt} implementation of the Transformer model \cite{DBLP:conf/nips/VaswaniSPUJGKP17}. Then for any input, in a non-trained model the logits $z_1,\ldots,z_d$ are distributed according to the normal distribution $\mathcal{N}\left(0,\frac{2\cdot d_{\text{model}}}{d_{\text{model}}+d_{\text{vocab}}}\right)$, where $d_{\text{model}}$ is the size of hidden representations, and $d_{\text{vocab}}$ is the  vocabulary size for a target language.
\end{lemma}
\begin{proof}
The default Transformer configuration in OpenNMT-py implies that the elements $w_{ij}$ of $\mathbf{W}$ are initialized from a uniform distribution $\mathcal{U}[-a,a]$, where $a=\sqrt{\frac{6}{d_{\text{model}}+d_{\text{vocab}}}}$, thus
\begin{equation}
    \E[w_{ij}]=0,\qquad\Var[w_{ij}]=\frac{(2a)^2}{12}=\frac{a^2}{3}=\frac{2}{d_{\text{model}}+d_{\text{vocab}}}\label{eq:w}
\end{equation}
Since $\mathbf{x}$ is the result of a layer normalization \cite{ba2016layer}, we have
\begin{equation}
    \frac{1}{d_{\text{model}}}\sum_{j=1}^{d_{\text{model}}} {x_j}=0,\qquad \frac{1}{d_{\text{model}}}\sum_{j=1}^{d_{\text{model}}}x_j^2=1\label{eq:x}
\end{equation}
Therefore, from \eqref{eq:w} and \eqref{eq:x}, we have
\begin{align*}
\E[z_i]&=\E\left[\sum_{j=1}^{d_{\text{model}}}w_{ij}x_j\right]=\sum_{j=1}^{d_{\text{model}}}\E[w_{ij}]\cdot x_j=0,\\
\Var[z_i]&=\Var\left[\sum_{j=1}^{d_{\text{model}}} w_{ij}x_j\right]=\frac{2}{d_{\text{model}}+d_{\text{vocab}}}\sum_{j=1}^{d_{\text{model}}}x_j^2=\frac{2\cdot d_{\text{model}}}{d_\text{model}+d_{\text{vocab}}}.
\end{align*}
Being a sum of independent random variables, by the Central Limit Theorem, each $z_i$ tends to normal distribution with the mean and variance above.
\end{proof}

\subsection{Derivation of the Equation~\eqref{eq:mde}}\label{app:dyn_proof}

We provide derivation for the case of $\alpha$-ReLU. Extension to $\alpha$-entmax and softmax is done analogously. Let $\mathbf{x}\in\mathbb{R}^{n_{0}}$ be the input vector. We define a feedforward neural network with $L-1$ hidden layers recursively:
\begin{align*}
    \mathbf{h}^{(0)}&=\mathbf{x}\\
    \mathbf{z}^{(k)}&=\frac{1}{\sqrt{n_{k-1}}}\mathbf{W}^{(k-1)}\mathbf{h}^{(k-1)},\\
    \mathbf{h}^{(k)}&=\sigma(\mathbf{z}^{(k)}),\quad k=1,\ldots,L-1
\end{align*}
where $\mathbf{W}^{(k-1)}\in\mathbb{R}^{{n_k}\times n_{k-1}}$ is the weight matrix in the $k^{\text{th}}$ hidden layer, and $\sigma(\cdot)$ is a nonlinear activation function applied element-wise. We consider the case of a multi-label classification, i.e. the output layer is a vector
$$
\mathbf{z}:=\mathbf{z}^{(L)}\in\mathbb{R}^d,
$$
which is fed into the $\alpha$-ReLU loss:
\begin{equation}
\ell(\mathbf{z},y)=(\arelu(\mathbf{z})-\mathbf{e}_y)^\top\left(\mathbf{z}-\frac{\tau}{\alpha-1}\mathbf{1}\right)\\+\h_\alpha[\arelu(\mathbf{z})],\label{eq:gen_loss}    
\end{equation}
where $\h_\alpha[\mathbf{p}]:=\frac{1}{\alpha(\alpha-1)}\sum_j(p_j-p_j^\alpha)$, $\alpha\ne1$, is the Tsallis $\alpha$-entropy \cite{tsallis1988possible}. Given a training sample $S:=\{({x},y)\}$ learning is performed by minimizing the training error
\begin{equation}
\mathcal{L}:=\E_{(x,y)\sim S}[\ell(\mathbf{z}(x),y)]\label{eq:train_err}
\end{equation}
with respect to the network parameters $\boldsymbol\theta:=\vect\left(\{\mathbf{W}^{(k-1)}\}_{k\in[L-1]}\right)$.
\begin{lemma}\label{lem:matrix_ode}
Let the training error \eqref{eq:train_err} be minimized by gradient descent with infinitesimally small learning rate. Let $\mathbf{z}(x,t)\in\mathbb{R}^d$ be the network output on any training instance $x$ at time $t$, and $y$ be the desired output. Then, as the widths of hidden layers $n_k\to\infty$, $\forall k\in[L-1]$, the output $\mathbf{z}(x,t)$ follows the following evolution
\begin{equation}
\frac{d\mathbf{z}}{dt}=-\underset{(x',y')\sim S}{\E}[\mathbf{K}(x,x')\cdot(\arelu(\mathbf{z}')-\mathbf{e}_{y'})],\label{eq:evo}
\end{equation}
where $\mathbf{K}(x,x')\in\mathbb{R}^{d\times d}$ is a positive semidefinite matrix, and $\mathbf{z}':=\mathbf{z}(x',t)$.
\end{lemma}
\begin{proof}
From \eqref{eq:train_err} and Lemma~\ref{lem:gradient} we have
\begin{equation}
\nabla_{\mathbf{z}}\mathcal{L}=\nabla_{\mathbf{z}}\E_{(x',y')\sim S}[\ell(\mathbf{z}',y)]=\nabla_{\mathbf{z}}\ell(\mathbf{z},y)=\arelu(\mathbf{z})-\mathbf{e}_y,\label{eq:gen_loss_grad}
\end{equation}
where we denoted $\mathbf{z}:=\mathbf{z}(x,t)$ and $\mathbf{z}':=\mathbf{z}(x',t)$ for shorthand. Now, consider the gradient descent update
\begin{equation}
    \boldsymbol\theta_{t+\eta}=\boldsymbol\theta_t-\eta\nabla_{\boldsymbol\theta}\mathcal{L}\quad
    \Leftrightarrow\quad \frac{\boldsymbol\theta_{t+\eta}-\boldsymbol\theta_t}{\eta}=-\nabla_{\boldsymbol\theta}\mathcal{L},\label{eq:gd_upd}
\end{equation}
where $\eta$ is the learning rate. Taking the limit in \eqref{eq:gd_upd} as $\eta\to0$, we have:
\begin{align*}
\frac{d\boldsymbol\theta}{dt}=-\nabla_{\boldsymbol\theta}\mathcal{L}=-\E_{(x',y')\sim S}[\mathbf{J}^\top_{\mathbf{z}'}(\boldsymbol\theta)\cdot\nabla_{\mathbf{z}'}\mathcal{L}],
\end{align*}
where the last equality is due to the chain rule. Combining this with \eqref{eq:gen_loss_grad}, we get
\begin{equation}
    \frac{d\boldsymbol\theta}{dt}=-\E_{(x',y')\sim S}[\mathbf{J}^\top_{\mathbf{z}'}(\boldsymbol\theta)\cdot(\arelu(\mathbf{z}')-\mathbf{e}_{y'})]\label{eq:theta_dyn}
\end{equation}
Applying the chain rule again, and using \eqref{eq:theta_dyn}, we have
\begin{equation}
\frac{d\mathbf{z}}{dt}=\mathbf{J}_{\mathbf{z}}(\boldsymbol\theta)\cdot\frac{d\boldsymbol\theta}{dt}\notag=-\E_{(x',y')\sim S}[\underbrace{\mathbf{J}_{\mathbf{z}}(\boldsymbol\theta)\mathbf{J}^\top_{\mathbf{z}'}(\boldsymbol{\theta})}_{\mathbf{K}(x,x';\boldsymbol\theta)}\cdot(\arelu(\mathbf{z}')-\mathbf{e}_{y'})]\label{eq:ntk}. 
\end{equation}
The quantity $\mathbf{K}(x,x';\boldsymbol\theta)$ was named the \emph{Neural Tangent Kernel} by \citet{DBLP:conf/nips/JacotHG18}. They also showed (see their Theorem~1) that 
$$
\mathbf{K}(x,x';\boldsymbol\theta)\to\mathbf{K}(x,x')\quad\text{as}\quad n_1,\ldots,n_{L-1}\to\infty,
$$
where $\mathbf{K}(x,x')\in\mathbf{R}^{d\times d}$ is the deterministric kernel that does not depend on $\boldsymbol\theta$. This concludes the proof.
\end{proof}

\section{Instructions for Human Annotators}\label{app:annot}
You are shown a reference sentence and several candidate translations. Please indicate, for each, on a 4-point scale, how much of the meaning is represented in the translation, ignoring the language quality.\\~\\
Imagine you are a forgiving reader, ignoring any error that does not prevent you from getting the meaning of the text.
So please ignore language oddities, typographic errors and the like. (This is difficult but key to us!)\\~\\
\textbf{The scale of \emph{meaning preservation} is: 4 = Everything   /  3 = Most  /  2 = Little  /  1 = None}\\~\\
As we are interested in comparing system's output, you can refine your judgement using $+$ or $-$, e.g. 3$+$.\\~\\
When you do not know, simply leave empty.\\~\\
For instance, given the reference sentence 
$$
\text{``This restaurant is beautiful and the staff is very friendly''},
$$
valid judgements for different translations are provided in Table~\ref{tab:eval_ex}.

\begin{table}[htbp]
\begin{center}
\begin{tabular}{l l}
\toprule
Score & Sentence \\
\midrule
  4 &  ``This restaurant is beautiful and the staff is very friendly.''\\
  4 &  ``This restaurant is beautiful and the staff is very friendly..''\\
  4 &  ``Beautiful restaurant, staff is very friendly.''\\
  4$-$ & ``This restaurant is beautiful and the staff is friendly.''\\
  4$-$ & ``Beautiful restaurant, staff is friendly.''\\
  2$+$ & ``Friendly staff''\\
  2  & ``This is a restaurant.''\\
  1  & ``Hello guys!''\\
  1  & ``Bad restaurant''\\
  1$-$ & ``Bad restaurant, bad staff'' \\
\bottomrule
\end{tabular}
\end{center}
\caption{Evaluation example}
\label{tab:eval_ex}
\end{table}

\noindent We insist that evaluating by meaning differs from a natural intuitive evaluation.
Provided the meaning is not impacted, we want to ignore the language quality, the punctuation, the casing.

\clearpage

\section{Translation Examples}\label{app:trans_examples}
\begin{table}[htbp]
    \centering
\begin{small}
    \begin{tabular}{l p{.4\textwidth} p{.4\textwidth}}
    \toprule
        Source & \multicolumn{2}{p{.8\textwidth}}{Flake was the central figure in \emph{Friday's} drama.} \\
        \hdashline
        1.5-entmax & Flèjk byl central'noj figuroj v drame \wrong{pjatnadcatogo veka}. & Flake was the central figure in \wrong{fifteenth century} drama.\\
        1.5-ReLU & Flèjk byl central'noj figuroj v drame \corr{pjatnicy}. & Flake was the central figure in \corr{Friday's} drama.\\
    \midrule
    Source & \multicolumn{2}{p{.8\textwidth}}{There were smiles and blue skies on Saturday (September 29) as the leaders of Turkey and Germany met for \emph{breakfast} in Berlin.} \\
    \hdashline
    1.5-entmax & V subbotu (29 sentjabrja) byli ulybki i goluboe nebo, poskol'ku lidery Turcii i Germanii vstretilis' dlja \wrong{razvala} v Berline. & There were smiles and blue skies on Saturday (September 29) as the leaders of Turkey and Germany met for a \wrong{breakup} in Berlin.\\
    1.5-ReLU & V subbotu (29 sentjabrja) byli ulybki i goluboe nebo, tak kak lidery Turcii i Germanii vstretilis' dlja \zh{otdyha} v Berline. & There were smiles and blue skies on Saturday (September 29) as the leaders of Turkey and Germany met for a \zh{holiday} in Berlin.\\
    \midrule
    Source & \multicolumn{2}{l}{That Was \emph{Really Bad} Body Language:} \\
    \hdashline
    1.5-entmax & Èto byl \corr{dejstvitel'no plohoj} jazyk tela & That was \corr{really bad} body language.\\
    1.5-ReLU & Èto byl \wrong{real'nyj} jazyk tela & That was \wrong{real} body language.\\
    \midrule
    Source & \multicolumn{2}{p{.8\textwidth}}{The city of Palu, which has more than 380,000 people, \emph{was strewn with debris from collapsed buildings}}\\
    \hdashline
    1.5-entmax & Gorod Palu, v kotorom prozhivaet bolee 380 000 chelovek, byl \corr{razrushen zdanijami.} & The city of Palu, home to over 380,000 people, was \corr{destroyed by buildings.} \\
    1.5-ReLU & Gorod Palu, u kotorogo bolee 380 000 chelovek, nahodilsja \wrong{v upadke zdanija.} & The city of Palu, which has over 380,000 inhabitants, was \wrong{in decay building.} \\
    \bottomrule
    \end{tabular}
\end{small}
    \caption{Translation Examples}
    \label{tab:trans_examples}
\end{table}

\end{document}